\documentclass[letterpaper, 10 pt, conference]{ieeeconf}  
\IEEEoverridecommandlockouts                 
\usepackage{graphicx}
\usepackage{amsmath,amssymb}
\usepackage{Cris}

\usepackage{ntheorem}

\newtheorem{proposition}{Proposition}
\newtheorem{assum}{Assumption}
\newtheorem*{problem}{Problem}
 
\usepackage{subfigure} 
\usepackage{algorithm}
\usepackage{algpseudocode}
\usepackage{cite} 
\usepackage{color}

\newcommand{\nquad}{\mkern-15mu}
\newcommand{\nsquad}{\mkern-10mu}
\DeclareMathOperator{\diag}{diag} 

\usepackage{nicefrac}
\usepackage{paralist}

 \graphicspath{{images/}}

\begin{document}
\title{\LARGE \bf
Achieving the Desired Dynamic Behavior in Multi-Robot Systems Interacting with the Environment}

\author{Lorenzo Sabattini, Cristian Secchi and Cesare Fantuzzi
\thanks{ 
Authors are with the Department of Sciences and Methods for Engineering (DISMI), University of Modena and Reggio Emilia, Italy {\tt\small{\{lorenzo.sabattini, cristian.secchi, cesare.fantuzzi\}@unimore.it}}}
}

\maketitle
\thispagestyle{empty}
\pagestyle{empty}

\begin{abstract}                          
In this paper we consider the problem of controlling the dynamic behavior of a multi-robot system while interacting with the environment. In particular, we propose a general methodology that, by means of locally scaling inter-robot coupling relationships, leads to achieving a desired interactive behavior. The proposed method is shown to guarantee passivity preservation, which ensures a safe interaction. The performance of the proposed methodology is evaluated in simulation, over large-scale multi-robot systems.
\end{abstract}

\section{Introduction} 

This paper proposes a general decentralized methodology for achieving a desired overall dynamic behavior for a multi-robot system interacting with the environment.

Typically, the behavior of a multi-robot system is defined by the interplay among basic control actions, such as aggregation, swarming, formation control, coverage and synchronization~\cite{ganguli2009,fax2004,olfatisaber2007,sabattiniauro2011}.  
%
%
 Modifying those basic actions makes it possible to change the characteristic properties of the overall multi-robot system. Along these lines, several methods can be found in the literature that tune the inter-robot coupling actions to modify global \emph{geometric properties} of the group~\cite{fax2004,ji2007,yang2008,xiao2009}, in terms of relative positions. 
 
%

Besides geometric properties, it is often of interest to regulate some \emph{topological properties} of the multi-robot systems, such as connectivity~\cite{ji2007,sabattiniijrr2013}, bi-connectivity~\cite{zareh2016,ahmadi2006}, controllability~\cite{sabattini2014} or rigidity~\cite{zelazo2012}.

In this paper we consider the problem of achieving a desired dynamic behavior when the multi-robot system is interacting with the environment, by means of appropriately tuning the coupling among neighboring robots. This problem was addressed in~\cite{urcola2011}, where an observation and estimation scheme was defined for understanding the behavior of humans. In particular, a common scaling factor was introduced, to reduce the inter-robot forces and, thus, impose constraints on the velocities and accelerations of the robots, when needed.

However, it is worth noting that uniformly scaling down the interaction forces among all the robots might lead to a too conservative solution, where connections among robots become too loose, and the primary objective of the multi-robot system can not be correctly fulfilled. 

In this paper we propose a strategy for achieving a desired interactive behavior of a multi-robot system with the environment. To this aim we will adopt a passivity based approach. In fact, guaranteeing the passivity of the multi-robot system is a sufficient condition for ensuring a stable behavior during the interaction with the, even poorly known, environment \cite{franchi2012}. Given a general passive cooperative nominal behavior of the multi-robot system, when a robot interacts with the environment, a desired interactive behavior is achieved by nonlinearly scaling the coupling with its neighbors. The proposed control of interaction is intended as a low-level control layer, to be coupled with some nominal control action, and is designed in a local manner, in order to affect as less as possible the nominal behavior of the multi-robot system.



The paper is organized as follows. Section~\ref{sec:notation} introduces the notation used in the paper. Problem formulation is provided in Section~\ref{sec:problem}. A method for tuning the coupling gains while preserving passivity is described in Section~\ref{sec:tunable}. This methodology is exploited in Section~\ref{sec:variation} for achieving the desired viscoelastic dynamic behavior. Simulations are described in Section~\ref{sec:simulations}, and concluding remarks are given in Section~\ref{sec:conclusions}.
%
%
%
%
%
%


\section{Notation}\label{sec:notation}
The symbol $\mathbb{I}_{m} \in \mathbb{R}^{m \times m}$ will indicate the identity matrix of dimension $m$, and the symbol $\mathbb{O}_{m,n}\in\mathbb{R}^{m \times n}$ will indicate the null matrix of dimension $m\times n$. For ease of notation, we will omit the dimension of the matrices when they appear clearly from the context.


Let $\Omega\in\mathbb{R}^{\rho\times\sigma}$ be a generic matrix. Then, we define $\Omega\left[i,j\right]\in\mathbb{R}$ as the element $\left(i,j\right)$ of $\Omega$. Moreover, the symbol $\otimes$ will be used to represent the Kronecker product.
\section{Problem Formulation}\label{sec:problem}
\newcounter{mytempeqncnt}
\newcounter{mytempeqncnt2}
\setcounter{mytempeqncnt2}{5}
\begin{figure*}[!t]
	\normalsize
	\setcounter{mytempeqncnt}{\value{equation}}
	\setcounter{equation}{\value{mytempeqncnt2}}
	\begin{equation}
	\label{eq:vectors}
	\resizebox{0.95\hsize}{!}{$
		\begin{array}[l]{lllll}
		F^e=
		\begin{pmatrix}
		F^e_1 \\ \vdots \\ F^e_N  
		\end{pmatrix} \in\calbR^{3N} & F^c=
		\begin{pmatrix}
		F^c_1 \\ \vdots \\ F^c_N  
		\end{pmatrix} \in\calbR^{3N} & p=
		\begin{pmatrix}
		p_1 \\ \vdots \\ p_N  
		\end{pmatrix} \in\calbR^{3N} 
		& \chi=\begin{pmatrix}
		\chi_1 \\ \vdots \\ \chi_{\bar{N}}
		\end{pmatrix} =
		\begin{pmatrix}
		x_{1,2} \\ \vdots \\ x_{N-1,N}  
		\end{pmatrix} \in\calbR^{3\bar{N}} 
		& \beta=\begin{pmatrix}
		\beta_1 \\ \vdots \\ \beta_{\bar{N}}
		\end{pmatrix} =
		\begin{pmatrix}
		\beta_{1,2} \\ \vdots \\ \beta_{N-1,N}  
		\end{pmatrix} \in\calbR^{3\bar{N}} 
		\end{array}
		$}
	\end{equation}
	\setcounter{equation}{\value{mytempeqncnt}}
	\hrulefill
\end{figure*}
Consider a system composed of $N$ robots moving in a three-dimensional environment, whose dynamics are modeled as follows: 
\begin{equation}
\label{eq:agent}
m_i\ddot{x}_i=w_i \quad i=1, \dots, N 
\end{equation}
where $x_i\in\mathbb{R}^3$ is the $i$-th robot's position,  $m_i>0$ is the $i$-th robot's mass, and $w_i\in\mathbb{R}^3$ collects control inputs and all the external forces each robot is subject to.

We consider the case where each robot is controlled in such a way that some desired cooperative behavior is achieved, while robots can interact with the environment. Hence, we consider the following generic input model: 
\begin{equation}
\label{eq:externalactions}
w_i=-\nquad\sum_{j=1, j\neq i}^N\nsquad\nabla V\left(x_{i,j}\right)-\nquad\sum_{j=1, j\neq i}^N\nsquad\beta_{i,j}\left(\dot{x}_i - \dot{x}_j\right)  +F_i^c  -b_i\dot{x}_i +F_i^e
\end{equation}
where $x_{i,j} = x_i - x_j$. 
 The terms in~\eqref{eq:externalactions} are defined as follows.
\setdefaultleftmargin{0.25cm}{}{}{}{}{}
\begin{enumerate}
	\item The term $-\nquad\displaystyle\sum_{j=1, j\neq i}^N\nsquad\nabla V\left(x_{i,j}\right)$ represents the \textbf{coupling} among robots. In particular, we consider each robot to interact with its neighbors, namely those robots whose distance is smaller than a certain threshold $R>0$, implementing a gradient descent of the artificial potential field $V\left(x_{i,j}\right) \geq 0$~\cite{sabattiniauro2011,leonard2001,franchi2012,falconi2015},
%
	 that has a global minimum at the desired inter-robot distance $\left\|x_{i,j}\right\| = \delta_d > 0$. The potential field is then designed in such a way that an attractive force is generated  if $\delta_d \leq \left\|x_{i,j}\right\| \leq R$, and a repulsive force is generated if $\left\|x_{i,j}\right\| < \delta_d $, such that the inter-robot distance does never go below the safety value $\delta_s$, with $0 < \delta_s < \delta_d$. A zero force is generated if two robots are too far away from each other, namely if $\left\|x_{i,j}\right\| > R$. 
	
	
	This kind of coupling represents an elastic interconnection among the robots, by means of nonlinear springs.
	
	According to this definition of the coupling potential, we define the $i$-th robot's neighborhood as follows:
	\begin{equation}
	\mathsf{N}_i = \left\{ j \neq i \text{ such that } \left\|x_{i,j}\right\| \leq R \right\}
	\label{eq:neighborhood}
	\end{equation}
	 %
	 %
	 %
	\item \textbf{Inter-robot damping} is represented by the term $-\nquad\displaystyle\sum_{j=1, j\neq i}^N\nsquad\beta_{i,j}\left(\dot{x}_i - \dot{x}_j\right)$, where $\beta_{i,j}$ is defined as follows:
	\begin{equation}
	\beta_{i,j} = \left\{\begin{array}{ll}
	\beta_{i,j}\geq 0  & \text{ if } \, j\in\mathsf{N}_i\\
	0 & \text{otherwise}
	\end{array}  \right.
	\end{equation}
	Together with the definition of the previously introduced coupling term, the design of $\beta_{i,j}$ leads to defining the overall desired behavior of the multi-robot system.
	\item The term $F_i^c$ represents an \textbf{additional control input} for the $i$-th robot, that can be utilized for achieving different objectives, such as imposing an offset~\cite{fax2004} or obtaining complex behaviors~\cite{sabattini2015}.
	\item The \textbf{local damping} term  $-b_i\dot{x}_i$, with $b_i>0$, represents both the viscous friction that characterizes the system and any additional damping injection obtained through a local control action. 
	\item The term $F_i^e$ represents the \textbf{interaction} force of the $i$-th robot \textbf{with the environment}. It can be either a real \emph{contact force}, measured by means of force sensors, or a \emph{virtual force}, generated by an obstacle avoidance artificial potential field~\cite{leonard2001,bouraine2012}.
\end{enumerate}



For ease of notation, we will hereafter define $V_{i,j} = V\left(x_{i,j}\right)$. Hence, considering the input defined in~\eqref{eq:externalactions}, the dynamics of the $i$-th robot introduced in~\eqref{eq:agent} can be rewritten as follows:
\begin{equation}
\label{eq:agent2}
m_i\ddot{x}_i+b_i\dot{x}_i+\nquad\sum_{j=1, j\neq i}^N\nsquad\nabla V_{i,j}+\nquad\sum_{j=1, j\neq i}^N\nsquad\beta_{i,j}\left(\dot{x}_i - \dot{x}_j\right)=F_i^e+F_i^c
\end{equation}

Besides defining how the robots coordinate among each other,~\eqref{eq:agent2} defines also \emph{how the multi-robot system interacts with the environment}: the coupling forces among the robots define the overall viscoelastic behavior of the multi-robot system. 

In this paper, we address the following problem:
\begin{problem}
	Define the coupling forces among the robots in such a way that the overall multi-robot system interacts with the environment with some desired viscoelastic behavior, while preserving its overall passivity.
\end{problem}



\section{Tuning of the Coupling Among the Robots While Preserving Passivity}\label{sec:tunable}

In this Section we will introduce a methodology for tuning the coupling among the robots while preserving passivity. For this purpose, we will rewrite the model of the multi-robot system in port-Hamiltonian form. Define then $p_i=m\dot{x}_i$ as the $i$-th robot's momentum, 
%
%
%
 let
 $\bar{N} = \frac{N(N-1)}{2}$, and consider the quantities defined in~\eqref{eq:vectors}. 
\stepcounter{equation}


Furthermore, let $\calI_{\calG}\in\mathbb{R}^{N \times \bar{N}}$ be the incidence matrix 
 of the complete graph among the robots\footnote{By \emph{complete graph} we refer to an undirected graph, in which each robot is represented by a node, and an edge exists among each pair of nodes.}. Define also $\bar{B}=\diag\left(\beta\right)$. The inter-agent damping term can then be modeled utilizing the weighted Laplacian matrix $\mathcal{L}_\beta \in \mathbb{R}^{N \times N}$ defined, as shown in~\cite{ji2007}, as $\mathcal{L}_\beta = \calI_{\calG}\, \bar{B}\, \calI_{\calG}^T$. 
Define now $M=\diag\left(m_1,\dots,m_N\right)$  and 
\begin{equation}
B = \mathcal{L}_\beta + \diag\left(b_1, \dots, b_N\right)
\label{eq:damping}
\end{equation}
as the inertia and damping matrix of the multi-robot system, respectively. Moreover, define  $\calI=\calI_{\calG}\otimes \mathbb{I}_3$.
%

The model of the multi-robot system can then be given in port-Hamiltonian form as follows:
\begin{equation}
  \label{eq:pHfleet}
  \resizebox{\hsize}{!}{$
  \left\{
    \begin{array}[l]{ll}
      \begin{pmatrix}
        \dot{p} \\ \dot{\chi}
      \end{pmatrix}=\left[
        \begin{pmatrix}
          \mathbb{O} & \calI \\-\cal I^T & \mathbb{O}
        \end{pmatrix}-
        \begin{pmatrix}
          B & \mathbb{O} \\ \mathbb{O} & \mathbb{O}
        \end{pmatrix}\right]
      \begin{pmatrix}
        \parder{H}{p} \\ \parder{H}{\chi}
      \end{pmatrix}+G(F^e+F^c) \\
v=G^T  \begin{pmatrix}
        \parder{H}{p} \\ \parder{H}{\chi}
      \end{pmatrix}
    \end{array}
\right.
$}
\end{equation}
where $v=\left(\dot{x}_1^T \ldots \dot{x}_N^T \right)^T\in\mathbb{R}^{3N}$ is the velocity vector, $G = \left( \mathbb{I}_{3N}\quad \mathbb{O}_{3\bar{N},3N}\right)^T$, 
%
%
 and $H$ is the total energy of the system given by:
\begin{equation}
  \label{eq:Htemp}
  H=\sum_{i=1}^N \calK_i(p_i)+\sum_{i=1}^{N}\sum_{i=j, j\neq i}^N V_{i,j} \geq 0
\end{equation}
where 
$\calK_i(p_i)=\nicefrac{p_i^Tp_i}{2m_i}$ is the kinetic energy associated to robot $i$. For ease of notation, define $V_k = V\left(\chi_k\right)$. Hence, it is possible to rewrite~\eqref{eq:Htemp} as:
\begin{equation}
  \label{eq:H}
  H=\sum_{i=1}^N \calK_i(p_i)+\sum_{k=1}^{\bar{N}}V_k \geq 0
\end{equation}

The following result can be trivially derived from~\cite[Proposition 1]{franchi2012}.
\begin{proposition}\label{prop:phfp}
Consider the dynamics of the multi-robot system described in port-Hamiltonian form in~\eqref{eq:pHfleet}, and consider the total energy of the system given in~\eqref{eq:H}. Then, the system is passive with respect to the pair $(F^c+F^e, v)$.
\end{proposition}
\begin{proof}
Consider the definition of the total energy of the system $H$ given in~\eqref{eq:H}. Then, considering the dynamics of the multi-robot system given in~\eqref{eq:pHfleet}, the time derivative of $H$ can be computed as follows:
  \begin{equation}
    \label{eq:phfp1}
    \begin{array}[l]{ll}
      \dot{H}=
    \begin{pmatrix}
      \parderT{H}{p} & \parderT{H}{\chi}
    \end{pmatrix}
    \begin{pmatrix}
      \dot{p} \\ \dot \chi
    \end{pmatrix}=    \begin{pmatrix}
      \parderT{H}{p} & \parderT{H}{\chi}
    \end{pmatrix}  \left[      \begin{pmatrix}
          0 & \calI \\-\cal I^T & 0
        \end{pmatrix}-\right.\\
\left.        \begin{pmatrix}
          B & 0 \\ 0 & 0
        \end{pmatrix}\right]
      \begin{pmatrix}
        \parder{H}{p} \\ \parder{H}{\chi}
      \end{pmatrix}+ \begin{pmatrix}
      \parderT{H}{p} & \parderT{H}{\chi}
    \end{pmatrix}   G(F^e+F^c) =\\=-\parderT{H}{p}B\parder{H}{p}+(F^e+F^c)^Tv\leq (F^e+F^c)^Tv
\end{array}
  \end{equation}
Thus,
\begin{equation}
  \label{eq:phfp2}
  \int_0^T(F^c+F^e)^Tvd\tau\geq H(t)-H(0)\geq-H(0)
\end{equation}
which completes the proof.
\end{proof}
Thus, the multi-robot system can safely interact with the environment, but its dynamic behavior is determined by the inter-robot coupling. We will hereafter define a methodology for \emph{scaling} the coupling forces among the robots, with the objective of achieving some desired dynamic viscoelastic behavior.

For ease of notation, we will hereafter make the following assumption:
\begin{assum}
For every time $t>0$, only one robot $i=1,\ldots, N$ exists such that the interaction force with the environment $F_i^e$ is different from zero. 
\end{assum}
Without loss of generality, in the rest of the paper we will always assume the $i$-th robot to be in contact with the environment at a given time $t$.

In order to modify the dynamic viscoelastic behavior in the interaction with the environment, we need to let the $i$-th robot adjust the force that couples it with its neighbors: this is possible by introducing a \emph{scaling factor}. It is worth noting that such a scaling affects the relation 
between the robots and the coupling actions (elastic and damping forces).

%
%
 
 Define now $A_{\calG}=\diag(\alpha_1,\dots, \alpha_{\bar{N}})$ as a diagonal matrix containing scaling factors, one per each pair of robots. If robot $i$ is interacting with the environment, it is sufficient to scale its interaction with the neighbors. Considering the definition of the $i$-th robot's neighborhood given in~\eqref{eq:neighborhood}, the elements of $A_{\calG}$ can be defined as follows:
%
%
\begin{equation}
  \label{eq:scalingfactors}
  \alpha_k=\left\{
    \begin{array}[l]{lll}
\alpha (t)>0      & \mbox{if } \calI_{\calG}\left[i,k\right] \neq 0 \text{ and } \left\|\chi_k\right\| \leq R\\ 
1 & \mbox{otherwise } 
    \end{array}
\right.
\end{equation}

In a similar manner, it is possible to introduce scaling factors for modulating the inter-agent damping. Define then $C_\mathcal{G} = \diag\left(c_1, \ldots, c_{\bar{N}}\right)$ as a diagonal matrix containing a scaling factor per each pair of robots. As in the case of $A_\calG$, define the elements of $C_\mathcal{G}$ as follows:
\begin{equation}
\label{eq:scalingfactorsdamp}
c_k=\left\{
\begin{array}[l]{lll}
\gamma (t)>0      & \mbox{if } \calI_{\calG}\left[i,k\right] \neq 0 \text{ and } \left\|\chi_k\right\| \leq R\\ 
1 & \mbox{otherwise } 
\end{array}
\right.
\end{equation}

In this way, only the edges transmitting coupling forces between $i$ and its neighbors are scaled. In particular, the matrix transmitting the scaled elastic forces to the robots is given by $(\calI_\calG A_\calG(t)) \otimes \mathbb{I}_3$. Exploiting the properties of the Kronecker product we have that:
\begin{multline}
\label{eq:kron}
(\calI_\calG A_\calG(t)) \otimes \mathbb{I}_3=(\calI_\calG A_\calG(t)) \otimes (\mathbb{I}_3\mathbb{I}_3)=\\=(\calI_\calG \otimes \mathbb{I}_3) (A_\calG(t)\otimes \mathbb{I}_3)=\calI A(t)
\end{multline}
where $A(t) = A_\calG(t)\otimes \mathbb{I}_3$. Moreover, the matrix transmitting the scaled viscous forces to the robots is given by
\begin{equation}
\underline{B} = \bar{\mathcal{L}}_\beta + \diag\left(b_1, \dots, b_N\right)
\label{eq:scaleddamping}
\end{equation}
where
\begin{equation}
\bar{\mathcal{L}}_\beta = \calI_{\calG} \left(C_\mathcal{G} \bar{B} \right) \calI_{\calG}^T
\end{equation}

Thus, the model of the multi-robot system~\eqref{eq:pHfleet} is modified as follows, with the introduction of the scaled coupling forces:
\begin{equation}
  \label{eq:scaledpHfleet}
  \resizebox{\hsize}{!}{$
  \left\{
    \begin{array}[l]{ll}
      \begin{pmatrix}
        \dot{p} \\ \dot{\chi}
      \end{pmatrix}=\left[
        \begin{pmatrix}
          \mathbb{O} & \calI A(t)\\-\cal I^T & \mathbb{O}
        \end{pmatrix}-
        \begin{pmatrix}
          \underline{B} & \mathbb{O} \\ \mathbb{O} & \mathbb{O}
        \end{pmatrix}\right]
      \begin{pmatrix}
        \parder{H}{p} \\ \parder{H}{\chi}
      \end{pmatrix}+G(F^e+F^c) \\
v=G^T  \begin{pmatrix}
        \parder{H}{p} \\ \parder{H}{\chi}
      \end{pmatrix}
    \end{array}
\right.
$}
\end{equation}
Comparing~\eqref{eq:scaledpHfleet} with~\eqref{eq:pHfleet}, it is possible to note that scaling of the inter-agent damping does not change the structure of the model: in fact, both  $\underline{B}$ and $B$ are positive definite: therefore, their role in the proof of \refprop{phfp} is analogous.

Conversely, due to the presence of the scaling matrix $A(t)$, the interconnection matrix in~\eqref{eq:scaledpHfleet} is not skew-symmetric as in~\eqref{eq:pHfleet}.
%
%
%
 As a consequence, the result of \refprop{phfp} can not be applied in this case. Intuitively, this is due to the fact that scaling the exchanged forces destroys the power balance among the robots. Nevertheless, we will hereafter show that passivity can still be guaranteed.
Considering time varying gains,~\eqref{eq:scaledpHfleet} can be rewritten as
\begin{equation}
  \label{eq:scaledpHfleet1}
  \resizebox{\hsize}{!}{$
  \left\{
    \begin{array}[l]{ll}
      \begin{pmatrix}
        \dot{p} \\ \dot{\chi}
      \end{pmatrix}=\left[
        \begin{pmatrix}
          \mathbb{O} & \calI \\-\cal I^T & \mathbb{O}
        \end{pmatrix}-
        \begin{pmatrix}
          \underline{B} & \mathbb{O} \\ \mathbb{O} & \mathbb{O}
        \end{pmatrix}\right]
      \begin{pmatrix}
        \parder{H}{p} \\  A(t)\parder{H}{\chi}
      \end{pmatrix}+G(F^e+F^c) \\
v=G^T  \begin{pmatrix}
        \parder{H}{p} \\ \parder{H}{\chi}
      \end{pmatrix}
    \end{array}
\right.
$}
\end{equation}
As a consequence, the energy function~\eqref{eq:H} is now modified as follows: 
\begin{equation}
  \label{eq:scaledH}
  H_s=\sum_{i=1}^N \calK_i(p_i)+\sum_{k=1}^{\bar{N}} \alpha_kV_{k} \geq 0 
\end{equation}

\begin{proposition}\label{prop:sphfp}
Consider the dynamics of the multi-robot system described in port-Hamiltonian form in~\eqref{eq:scaledpHfleet1}, and consider the total energy of the system given in~\eqref{eq:scaledH}. Then, if $\exists \alpha_m, \alpha_M \in\mathbb{R}$, $0<\alpha_m<\alpha_M$, such that $\alpha_m\leq \alpha_k(t) \leq \alpha_M$ for any time $t\geq 0$, then the system is passive with respect to the pair $(F^c+F^e, v)$.		
%
%
\end{proposition}
\begin{proof}
  Consider the following energy function:
  \begin{equation}
    \label{eq:sphfp1}
    H_s(t)= \sum_{i=1}^N \calK_i(p_i)+\sum_{k=1}^{\bar{N}}\alpha_k(t)V_{k}
  \end{equation}
Following the same steps taken in the proof of \refprop{phfp}, we have that 
\begin{equation}
  \label{eq:sphfp3}
  \dot{H}_s=(F^c+F^e)^Tv-\parderT{H_s}{p}\underline{B}\parder{H_s}{p}+\parder{H}{t}
\end{equation}
and consequently, by integrating and by reminding that $\underline{B}$ is positive definite:
\begin{equation}
  \label{eq:sphfp4}
  H_s(t)-H_s(0)\leq\int_0^t(F^c+F^e)^Tvd\tau+\int_0^t\parder{H}{\tau}d\tau
\end{equation}
%
Integrating by parts we get
\begin{equation}
  \label{eq:sphfp6}
  \begin{array}[l]{ll}
\displaystyle \int_0^t\parder{H}{\tau}d\tau=\!\!\int_0^t\sum_{k=1}^{\bar{N}}\dot{\alpha}_k(\tau)V_k(\tau)d\tau=\\ =\!\displaystyle \sum_{k=1}^{\bar{N}}\alpha_k(t)V_k(t)-\!  \sum_{k=1}^{\bar{N}}\alpha_k(0)V_k(0)-\!\int_0^t\! \sum_{k=1}^{\bar{N}}\alpha_k(\tau)\dot{V}_k(\tau)d\tau
  \end{array}
\end{equation}

For any time interval $\left[0,T\right]$, and for any $k=1, \ldots, \bar{N}$, it is possible to define the following subsets:
\begin{equation}
\begin{aligned}
\Phi_k^+ = \left\{t \in \left[0,T\right] \text{ such that } \dot{V}_k(t) \geq 0  \right\}\\
\Phi_k^- = \left\{t \in \left[0,T\right] \text{ such that } \dot{V}_k(t) < 0  \right\}
\end{aligned}
\label{eq:timeset}
\end{equation}
The subsets $\Phi_k^+$ and $\Phi_k^-$ are the union of a finite number of disjoint time intervals\footnote{Pathological situations might exist in which the number of time intervals that compose  $\Phi_k^+$ and $\Phi_k^-$ is not finite. However, in practical situations this does never happen.}. Let $T_k^+$ and $T_k^-$ be the number of time intervals that compose $\Phi_k^+$ and $\Phi_k^-$, respectively. Let $\underline{t}_{k,h}^+, \bar{t}_{k,h}^+$ be the initial and final times of the $h$-th time interval of  $\Phi_k^+$, and let $\underline{t}_{k,h}^-, \bar{t}_{k,h}^-$ be the initial and final times of the $h$-th time interval of  $\Phi_k^-$. Then, the subsets can be defined as follows:
\begin{equation}
\begin{aligned}
\Phi_k^+ = \left[\underline{t}_{k,1}^+, \bar{t}_{k,1}^+\right] \bigcup \ldots \bigcup \left[\underline{t}_{k,T_k^+}^+, \bar{t}_{k,T_k^+}^+\right] \\
\Phi_k^- = \left[\underline{t}_{k,1}^-, \bar{t}_{k,1}^-\right] \bigcup \ldots \bigcup \left[\underline{t}_{k,T_k^-}^-, \bar{t}_{k,T_k^-}^-\right]
\end{aligned}
\label{eq:timeintervals}
\end{equation}
where
\begin{equation}
\begin{aligned}
0 \leq \underline{t}_{k,1}^+ < \bar{t}_{k,1}^+ < \ldots < \underline{t}_{k,T_k^+}^+ < \bar{t}_{k,T_k^+}^+ \leq T  \\
0 \leq \underline{t}_{k,1}^- < \bar{t}_{k,1}^- < \ldots < \underline{t}_{k,T_k^-}^- < \bar{t}_{k,T_k^-}^- \leq T
\end{aligned}
\label{eq:sequencetimes}
\end{equation}
%
%
Hence, we can rewrite the integral in~\eqref{eq:sphfp6} as follows:
\begin{equation}
\resizebox{\hsize}{!}{$
\begin{array}{l}
\displaystyle\int_0^t \sum_{k=1}^{\bar{N}}\alpha_k(\tau)\dot{V}_k(\tau)d\tau =\\ \displaystyle \sum_{k=1}^{\bar{N}}\left( \int_{\Phi_k^+} \alpha_k(\tau)\dot{V}_k(\tau)d\tau +  \int_{\Phi_k^-} \alpha_k(\tau)\dot{V}_k(\tau)d\tau \right) =\\
\displaystyle \sum_{k=1}^{\bar{N}}\left(\sum_{h=1}^{T_k^+} \int_{\underline{t}_{k,h}^+}^{\bar{t}_{k,h}^+} \alpha_k(\tau)\dot{V}_k(\tau)d\tau + \sum_{h=1}^{T_k^-} \int_{\underline{t}_{k,h}^-}^{\bar{t}_{k,h}^-} \alpha_k(\tau)\dot{V}_k(\tau)d\tau\right)
\end{array}
$}
\label{eq:integrewr}
\end{equation}
Since $0\leq \alpha_m \leq \alpha(t) \leq \alpha_M$, according to the definition given in~\eqref{eq:timeintervals}, it is possible to obtain the following inequality:
\begin{equation}
\begin{array}{l}
-\displaystyle\int_0^t \sum_{k=1}^{\bar{N}}\alpha_k(\tau)\dot{V}_k(\tau)d\tau \leq\\ 
\displaystyle -\sum_{k=1}^{\bar{N}}\left(\alpha_m\sum_{h=1}^{T_k^+} \int_{\underline{t}_{k,h}^+}^{\bar{t}_{k,h}^+} \dot{V}_k(\tau)d\tau + \alpha_M \sum_{h=1}^{T_k^-} \int_{\underline{t}_{k,h}^-}^{\bar{t}_{k,h}^-} \dot{V}_k(\tau)d\tau\right)
\end{array}
\label{eq:ineqintegr}
\end{equation}
Hence, from~\eqref{eq:sphfp6} we obtain the following:
\begin{equation}
\label{eq:sphfp6b}
\begin{array}[l]{ll}
\displaystyle \int_0^t\parder{H}{\tau}d\tau\leq  \sum_{k=1}^{\bar{N}}\alpha_k(t)V_k(t)-\sum_{k=1}^{\bar{N}}\alpha_k(0)V_k(0)-\\
\displaystyle -\sum_{k=1}^{\bar{N}}\left(\alpha_m\sum_{h=1}^{T_k^+} \int_{\underline{t}_{k,h}^+}^{\bar{t}_{k,h}^+} \dot{V}_k(\tau)d\tau + \alpha_M \sum_{h=1}^{T_k^-} \int_{\underline{t}_{k,h}^-}^{\bar{t}_{k,h}^-} \dot{V}_k(\tau)d\tau\right)
\end{array}
\end{equation}
%
%
%
%
%
Thus, from~\eqref{eq:sphfp4}, the following inequality can be derived:
\begin{equation}
\resizebox{\hsize}{!}{$
\begin{array}[l]{ll}
H_s(t)-H_s(0)\leq \displaystyle\int_0^t(F^c+F^e)^Tvd\tau+\sum_{k=1}^{\bar{N}}\alpha_k(t)V_k(t)-\\
\displaystyle-\sum_{k=1}^{\bar{N}}\alpha_k(0)V_k(0)- \alpha_m\sum_{k=1}^{\bar{N}}\sum_{h=1}^{T_k^+}V_k\left(\bar{t}_{k,h}^+ \right) +\alpha_m\sum_{k=1}^{\bar{N}}\sum_{h=1}^{T_k^+}V_k\left(\underline{t}_{k,h}^+ \right)-\\
\displaystyle - \alpha_M\sum_{k=1}^{\bar{N}}\sum_{h=1}^{T_k^-}V_k\left(\bar{t}_{k,h}^- \right) +\alpha_M\sum_{k=1}^{\bar{N}}\sum_{h=1}^{T_k^-}V_k\left(\underline{t}_{k,h}^- \right)
\end{array}
$}
\label{eq:sphfp8}
\end{equation}
Hence, considering the definition of $H_s(t)$ given in~\eqref{eq:sphfp1}, it is possible to obtain the following inequality:
%
%
%
%
%
%
%
%
\begin{equation}
\label{eq:sphfp10}
\begin{array}[l]{ll}
\displaystyle\int_0^t(F^c+F^e)^Tvd\tau\geq\sum_{i=1}^N(K_i(t)-K_i(0))+\\
\displaystyle+\alpha_m\sum_{k=1}^{\bar{N}}\sum_{h=1}^{T_k^+}\left(V_k\left(\bar{t}_{k,h}^+ \right) -V_k\left(\underline{t}_{k,h}^+ \right)\right)+\\
\displaystyle + \alpha_M\sum_{k=1}^{\bar{N}}\sum_{h=1}^{T_k^-}\left(V_k\left(\bar{t}_{k,h}^- \right) - V_k\left(\underline{t}_{k,h}^- \right) \right)
\end{array}
\end{equation}
%
%
Since both $\alpha_m$ and $\alpha_M$ are positive, and since both the kinetic energy $K_i\left(\cdot\right)$ and the potentials $V_k\left(\cdot\right)$ are positive, it is possible to obtain the following:
\begin{equation}
\label{eq:sphfp11}
	\displaystyle\int_0^t(F^c+F^e)^Tvd\tau\geq -\sum_{i=1}^NK_i(0)
\end{equation}
%
%
%
which proves the passivity.
\end{proof}
\section{Tunable Interaction with the Environment}\label{sec:variation}

In this Section we will show how to utilize the methodology introduced so far for adjusting the parameters of the inter-robot coupling, in a local manner, in order to achieve a desired viscoelastic dynamic behavior in the interaction with the environment. We will hereafter assume that the robots' mass is sufficiently small, such that inertial forces can be neglected. Hence, 
%
 the force robot $i$ applies to the environment is equal to 
\begin{equation}
F_i=\alpha\left(t\right)\nquad \sum_{j=1, j\neq i}^{N}\nsquad \nabla V_{i,j}+F^c_i + \gamma\left(t\right)\nquad\sum_{j=1, j\neq i}^{N}\nsquad\beta_{i,j}\left(\dot{x}_i - \dot{x}_j\right)+ b_i\dot{x}_i
\label{eq:forcei}
\end{equation} 
Since the elastic coupling term between any two robots is only a function of their relative positions, it is always possible to write it as follows:
\begin{equation}
\nabla V_{i,j} = \kappa_{i,j}\left(x_i,x_j\right)\,\left(x_i - x_j\right)
\end{equation}
For ease of notation, we will hereafter omit the dependency of $\kappa_{i,j}$ on $x_i,x_j$.  Hence,~\eqref{eq:forcei} can be rewritten as follows:
\begin{equation}
\resizebox{\hsize}{!}{$
F_i\!=\! \alpha\left(t\right)\nquad \displaystyle\sum_{j=1, j\neq i}^{N}\nsquad \kappa_{i,j}\left(x_i - x_j\right)+F^c_i + \gamma\left(t\right)\nquad\displaystyle\sum_{j=1, j\neq i}^{N}\nsquad\beta_{i,j}\left(\dot{x}_i - \dot{x}_j\right)+ b_i\dot{x}_i
$}
\label{eq:forcei2}
\end{equation}

This force can be modeled as a single standard viscoelastic force as follows:
\begin{equation}
F_i = \kappa_n\left(x_i-\bar{x}\right) + F^c_i + \beta_n\left(\dot{x}_i - \bar{v}\right) + b_i\dot{x}_i
\label{eq:viscoelastic}
\end{equation}
Let $\Delta$ be a constant representing the rest-length of a standard elastic element, defined based on the application. Then, the nominal stiffness $\kappa_n$ and rest position $\bar{x}$ of the spring are defined as follows:
\begin{equation}
\left( x_i-\bar{x} \right) = \pm \Delta\! {\displaystyle \sum_{j=1, j\neq i}^{N}\nsquad \kappa_{i,j}\left(x_i - x_j\right)}\Big/{\left\| \displaystyle\sum_{j=1, j\neq i}^{N}\nsquad \kappa_{i,j}\left(x_i - x_j\right) \right\|}
\label{eq:ximinusbarx}
\end{equation}
\begin{equation}
\kappa_n = {\left\|\alpha\left(t\right)\nquad \displaystyle \sum_{j=1, j\neq i}^{N}\nsquad \kappa_{i,j}\left(x_i - x_j\right)\right\|}\Big/{\left\|x_i - \bar{x}\right\|}
\label{eq:kappan}
\end{equation}
The value of $\bar{x}$ is then defined according to~\eqref{eq:ximinusbarx} in such a way that the following holds:
\begin{equation}
\kappa_n\left(x_i-\bar{x}\right) = \alpha\left(t\right)\nquad \displaystyle\sum_{j=1, j\neq i}^{N}\nsquad \kappa_{i,j}\left(x_i - x_j\right)
\end{equation}
Parameters $\bar{v}$ and $\beta_n$ are defined in an analogous manner.

It is worth noting that $\bar{x}$ plays the role of the desired position for robot $i$ in stiffness control. Hence, it is possible to define a desired dynamic behavior for the multi-robot system, in terms of a desired viscoelastic dynamics, as follows:
\begin{equation}
F_d = \kappa_d\left(x_i-\bar{x}\right) + F^c_i + \beta_d\left(\dot{x}_i - \bar{v}\right) + b_i\dot{x}_i
\label{eq:desireddyn}
\end{equation} 
for some desired $\kappa_d>0$, $\beta_d>0$. As discussed in Section~\ref{sec:tunable}, the damping coefficient can be freely adjusted with an appropriate choice of $\gamma\left(t\right)$. Therefore, imposing $\beta_n = \beta_d$, we obtain
\begin{equation}
\gamma\left(t\right)= {\beta_d\left\|\dot{x}_i - \bar{v}\right\|}\Big/{\left\|\displaystyle\sum_{j=1, j\neq i}^{N}\nsquad\beta_{i,j}\left(\dot{x}_i - \dot{x}_j\right)\right\|}
\label{eq:gamma}
\end{equation}
Conversely, 
%
  the desired elastic stiffness can be achieved exploiting the results of Proposition~\ref{prop:sphfp}. 

In particular, it is possible to tune the coupling between the $i$-th robot and its neighbors utilizing the parameter $\alpha(t)$: the objective is that of minimizing the difference between $F_i$ and $F_d$. Namely, considering the elastic terms in~\eqref{eq:viscoelastic} and~\eqref{eq:desireddyn}, it is necessary to minimize the following cost function:
\begin{equation}
f\left(\alpha\right) = \left(\kappa_d\left(x_i-\bar{x}\right) - \alpha\nquad \sum_{j=1, j\neq i}^{N}\nsquad \kappa_{i,j}\left(x_i - x_j\right) \right)^2
\label{eq:costfun}
\end{equation}
Hence, we can define the following simple quadratic optimization problem:
\begin{equation}
\begin{array}{ll}
\text{minimize} \, & f\left(\alpha\right) \\
\text{subject to} \, & \alpha_m \leq \alpha \leq \alpha_M
\end{array}
\label{eq:opt_prob}
\end{equation}
where $\alpha_M>\alpha_m > 0$ are the upper- and lower-bounds for $\alpha(t)$, respectively, defined according to Proposition~\ref{prop:sphfp}.


In a decentralized multi-robot system, decisions are taken in a local manner, without any centralized elaboration unit. As a consequence, if the $i$-th robot is in contact with the environment, it will locally solve the optimization problem in~\eqref{eq:opt_prob} and find the desired value $\alpha^\star$. It is then necessary for the $i$-th robot to broadcast this value to its neighbors, in such a way that the coupling actions can be tuned as required.

This procedure needs to be performed as soon as the $i$-th robot perceives an interaction force with the environment. Furthermore, the desired value $\alpha^\star$ needs to be periodically recomputed, every $\bar{T}>0$ seconds, based on the current force measurements. 

\section{Simulations}\label{sec:simulations}
In this Section we describe the results of the evaluation of the proposed control method. In particular, several simulations were performed in an environment developed in MATLAB\textsuperscript{\textregistered}. Specifically, a variable number of three-dimensional double integrator robots, modeled according to~\eqref{eq:agent}, was considered. Let the three-dimensional environment be defined by the $\left(\mathsf{x,y,z}\right)$ axes.

We utilized the coupling artificial potential field defined in~\cite{leonard2001}, 
 with the following parameter set: $\delta_s = 5$, $\delta_d = 15$, $R = 22$. The additional control input $F_i^c\in\mathbb{R}^3$ was utilized for imposing a motion of the group of robots along the $\mathsf x$ axis. 
%
 A point obstacle was then placed in the environment, and a repulsive artificial potential field was activated for those robots whose distance from the obstacle was smaller than $\delta_d$, in order to guarantee that the distance remained larger than $\delta_s$.

 For each simulation run, the number $N$ of robots was defined, and initial positions were randomly chosen.  Since, as detailed in Section~\ref{sec:variation}, the parameter $\gamma\left(t\right)$ can be arbitrarily chosen to achieve the desired damping, the evaluation focused on the stiffness tuning. Then, two cases were considered:
\begin{enumerate}
	\item The \emph{nominal} case, where we utilized a constant coupling gain $\alpha\left(t\right)=30$.
	\item The \emph{tunable stiffness} case, where we utilized the desired stiffness $\kappa_d = 1$. In this case, the optimization problem~\eqref{eq:opt_prob} was defined with $\alpha_m=10^{-4}$, and $\alpha_M = 10^2$.
\end{enumerate}

The results of a representative simulation run are depicted in Fig.~\ref{fig:sim_64}, where we utilized, $N=64$ robots: the tunable stiffness case is compared to the nominal one.

The value of the cost function $f\left(\alpha\right)$ defined in~\eqref{eq:costfun} is depicted in Fig.~\ref{fig:64_cf}. 
 The cost function was evaluated only when a robot was in contact with the obstacle. As expected, the value of $f\left(\alpha\right)$ is typically very large for the nominal case, while it becomes very small for the tunable stiffness case. 

In order to evaluate the deviation from the nominal behavior, we measured the position of the barycenter of the multi-robot system. The percentage deviation in the tunable stiffness case with respect to the nominal one is depicted in Fig.~\ref{fig:64_bary} where, due to space limitations, only the component along the $\mathsf{x}$ axis is depicted. It is possible to note that large deviations are observed only during limited periods of time, in particular when a robot is in contact with the obstacle. 

The accompanying video shows a few examples of simulation runs, where different numbers of robots were utilized. It is possible to note that, in the nominal case, the multi-robot system behaves as a rigid body, while in the tunable stiffness case the coupling between the robots is locally weakened, when in contact with the obstacle, in order to achieve the desired dynamic behavior.

\begin{figure}[bt]
	\centering
	\subfigure[Cost function $f\left(\alpha\right)$ defined in~\eqref{eq:costfun}: red solid line for the nominal case, green dashed line for the tunable stiffness case]{\includegraphics[width=\columnwidth]{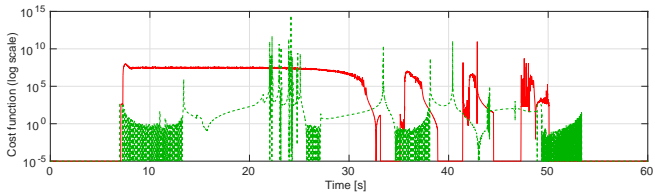}\label{fig:64_cf}}
	\subfigure[Barycenter position along the x axis: percentage deviation in the tunable stiffness case]{\includegraphics[width=\columnwidth]{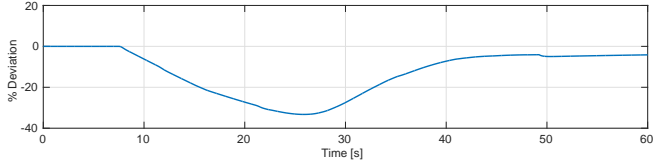}\label{fig:64_bary}}
	\caption{Simulation performed with 64 robots}%
	\label{fig:sim_64}%
\end{figure}



\section{Conclusions}\label{sec:conclusions}

In this paper we addressed the problem of controlling the dynamic behavior of a multi-robot system interacting with the environment. In particular, we proposed a general methodology that, introducing a local scale factor on the inter-robot couplings, leads to achieving a desired overall viscoelastic dynamic behavior. 
 The proposed method was shown to guarantee passivity preservation, thus ensuring a safe interaction. 


Throughout the paper, we assumed only one robot at a time to be in contact with the environment. Future work will aim at relaxing this assumption, considering multiple robots simultaneously in contact with the environment. 

\bibliographystyle{IEEEtran}
\bibliography{biblio}
\end{document}